\documentclass[11pt]{article}

\usepackage[margin=1in]{geometry}
\usepackage{natbib}

\usepackage[utf8]{inputenc}
\usepackage[T1]{fontenc}
\usepackage{hyperref}
\usepackage{url}
\usepackage{booktabs}
\usepackage{amsfonts}
\usepackage{amsmath}
\usepackage{amssymb}
\usepackage{amsthm}
\usepackage{graphicx}
\usepackage{xcolor}
\usepackage{colortbl}
\usepackage{subcaption}
\usepackage{algorithm}
\usepackage{algorithmic}
\usepackage{bbm}
\usepackage{nicefrac}
\usepackage{wrapfig}

\newtheorem{theorem}{Theorem}
\newtheorem{proposition}[theorem]{Proposition}
\newtheorem{corollary}[theorem]{Corollary}
\newtheorem{lemma}[theorem]{Lemma}

\newtheorem{remark}[theorem]{Remark}

\newcommand{\R}{\mathbb{R}}
\newcommand{\E}{\mathbb{E}}

\newcommand{\dd}{\mathrm{d}}

\title{On the Role of Strain and Vorticity in \\
Numerical Integration Error for Flow Matching}

\author{
  \textbf{Chenxi Tao} \\
  Georgia Institute of Technology \\
  \texttt{ctao40@gatech.edu}
  \and
  \textbf{Seung-Kyum Choi} \\
  Georgia Institute of Technology \\
  \texttt{schoi@me.gatech.edu}
}

\date{} % Leave empty if you don't want today's date to appear automatically

\begin{document}

\maketitle

% ============================================================
\begin{abstract}
% ============================================================
Flow matching generates data by integrating a learned velocity field, where the number of integration steps (NFE) directly determines inference cost. Yet a precise understanding of \emph{which properties of the velocity field govern integration error} has been lacking. We provide such an understanding by decomposing the velocity Jacobian $\nabla_x v$ into its symmetric part $S$ (strain rate) and antisymmetric part $\Omega$ (vorticity), and proving that they play fundamentally different roles: strain controls \emph{exponential} error amplification via the logarithmic norm $\mu_2 = \lambda_{\max}(S)$, while vorticity contributes only \emph{linearly} to the local truncation error. This asymmetry has three implications. First, we derive a separated error bound showing that suppressing strain alone eliminates exponential error growth, while suppressing vorticity alone does not. Second, we prove that the optimal transport velocity field is automatically irrotational ($\Omega = 0$) and has zero material derivative, which upgrades Euler integration from first-order to second-order accuracy. For exact OT displacement interpolation, the corresponding Lagrangian particle dynamics are in fact integrated exactly by Euler; we verify this on both Gaussian and nonlinear OT flows, where errors reach machine precision ($\sim\!10^{-14}$). Third, we show that a weighted Jacobian regularizer with strain weight $\alpha$ exceeding vorticity weight $\beta$ is theoretically favored, a prediction we validate on synthetic benchmarks and probe on CIFAR-10. Experiments on 2D distributions confirm the main theoretical predictions, demonstrating up to $2.7\times$ reduction in integration error at NFE$=$5. Preliminary CIFAR-10 experiments show consistent trends, with a lightweight fine-tuning procedure yielding 14\% FID improvement at NFE$=$10 while preserving high-NFE quality. A matched fine-tuning control experiment (same training, no regularization) shows no comparable improvement, indicating that the gains are associated with Jacobian regularization rather than additional training alone. Ablations further illustrate the predicted bias-complexity tradeoff and support strain-dominant weighting in the low-dimensional setting.
\end{abstract}

% ============================================================
\section{Introduction}
\label{sec:intro}
% ============================================================

Flow Matching~\citep{lipman2023flow, liu2023rectified} has emerged as a powerful paradigm for generative modeling, training a velocity field $v_\theta(t, x)$ whose ODE integration transports noise to data. A central practical challenge is that accurate integration requires many function evaluations (high NFE), making inference slow. A variety of methods address this problem: Rectified Flow~\citep{liu2023rectified} straightens trajectories via reflow, Consistency Models~\citep{song2023consistency} enforce self-consistency, and MeanFlow~\citep{geng2025meanflow} learns an average velocity enabling one-step generation.

Despite these practical advances, a fundamental question remains: \emph{what properties of the learned velocity field determine how many integration steps are needed?} Standard numerical analysis bounds the Euler error using the Lipschitz constant $L = \sup\|\nabla_x v\|$, yielding $O(h \cdot e^{LT})$. But this bound treats all components of the Jacobian equally, potentially missing structure that could yield tighter analysis and better-targeted regularization.

In this paper, we provide a finer-grained analysis by decomposing the velocity Jacobian into its \textbf{symmetric part} $S$ (strain rate tensor) and \textbf{antisymmetric part} $\Omega$ (vorticity tensor). We prove that these two components affect integration error in fundamentally different ways:

\begin{itemize}
    \item \textbf{Strain controls exponential error amplification.} The error propagation factor is governed by the logarithmic norm $\mu_2(\nabla v) = \lambda_{\max}(S)$, which depends \emph{only} on $S$. Large strain causes errors to grow as $e^{\mu_+ T}$.
    
    \item \textbf{Vorticity contributes only linearly.} Vorticity affects the local truncation error through the term $\Omega v$, but this contribution is purely additive --- it does not enter the exponential amplification factor.
    
    \item \textbf{Suppressing both yields the tightest Euler bound.} When $S \to 0$ and $\Omega \to 0$ (the vanishing-strain-and-vorticity regime), the error collapses from $O(h \cdot e^{LT})$ to $O(hT \cdot M_t)$, where $M_t = \sup\|\partial_t v\|$.
\end{itemize}

We further connect this analysis to optimal transport theory, proving that the OT velocity field from Brenier's theorem is irrotational ($\Omega = 0$) and has zero material derivative ($Dv/Dt = 0$). The latter implies that Euler integration is automatically \emph{second-order} accurate on OT flows.

\textbf{Contributions.}
\begin{enumerate}
    \item A \textbf{separated error bound} (Theorem~\ref{thm:main}) proving the asymmetric roles of strain and vorticity in ODE integration error, together with a three-regime analysis (Corollary~\ref{cor:regimes}).
    \item Proof that the \textbf{OT velocity field is irrotational} (Theorem~\ref{thm:ot-irrotational}) and has \textbf{zero material derivative} (Theorem~\ref{thm:ot-second-order}), which yields second-order Euler accuracy in the Eulerian error analysis. For exact displacement interpolation, we further observe exact Lagrangian Euler integration on both Gaussian and nonlinear OT flows, a stronger phenomenon discussed in Remark~\ref{rem:exact-lagrangian}.
    \item A \textbf{bias-complexity tradeoff analysis} showing that strain regularization is theoretically more valuable than vorticity regularization for controlling Euler discretization error (Proposition~\ref{prop:alpha-beta}).
    \item \textbf{Experimental validation} on 2D benchmarks confirming the main theoretical predictions, together with supporting CIFAR-10 experiments that show consistent low-NFE improvements and ablations over $\alpha$, $\beta$, and fine-tuning duration.
\end{enumerate}

Figure~\ref{fig:hero} illustrates the core idea: as both strain and vorticity are suppressed, particle trajectories become progressively straighter and non-crossing, enabling accurate integration with fewer steps.

\begin{figure}[t]
\centering
\includegraphics[width=\textwidth]{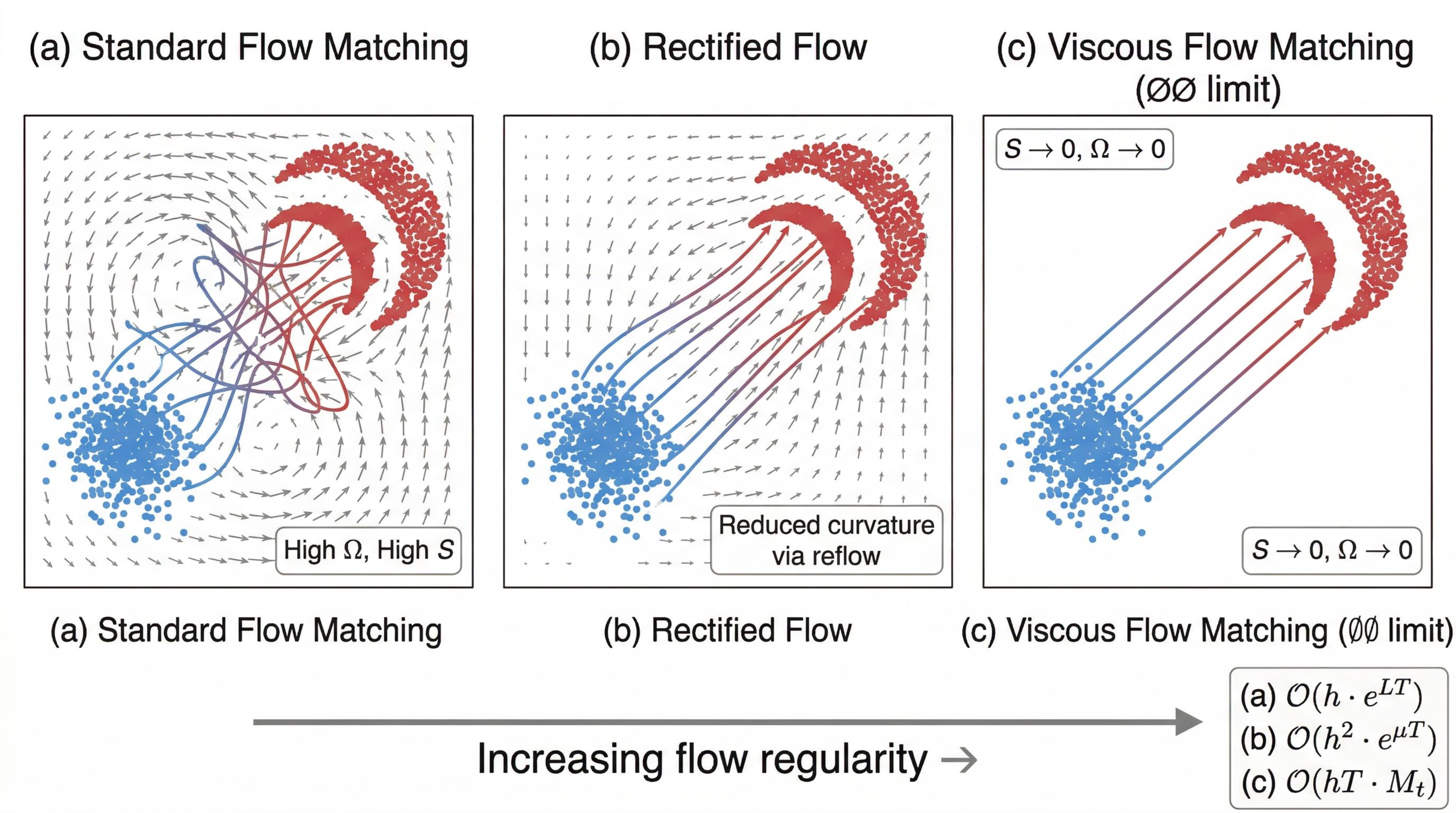}
\caption{Overview of the three flow regimes. (a)~Standard FM: chaotic, crossing trajectories with high strain and vorticity, requiring many integration steps. (b)~Vorticity suppressed ($\Omega \to 0$): smoother but still curved trajectories; exponential error amplification persists due to strain. (c)~Double-null limit ($S, \Omega \to 0$): nearly straight, parallel trajectories; error growth is linear, enabling few-step generation. Error bounds from Theorem~\ref{thm:main} and Corollary~\ref{cor:regimes} are shown below each panel.}
\label{fig:hero}
\end{figure}

% ============================================================
\section{Background}
\label{sec:background}
% ============================================================

\textbf{Flow Matching.} Given source distribution $p_0 = \mathcal{N}(0, I)$ and data distribution $p_1$, flow matching~\citep{lipman2023flow} learns $v_\theta: [0,1] \times \R^d \to \R^d$ by minimizing:
$\mathcal{L}_{\text{FM}}(\theta) = \E_{t, x_0, x_1} \| v_\theta(t, x_t) - u_t(x_t | x_1) \|^2$,
where $x_t = (1 - (1-\sigma)t)x_0 + t x_1$ is the conditional OT interpolation.

\textbf{Euler Integration Error.} At inference, we solve $\dd x / \dd t = v_\theta(t, x)$ with step size $h = 1/N$:
$\hat{x}_{n+1} = \hat{x}_n + h \cdot v(t_n, \hat{x}_n)$.
The standard error bound is $\|e_N\| = O(h \cdot (e^{LT} - 1)/L)$, where $L = \sup_t \|\nabla_x v(t, \cdot)\|$ is the Lipschitz constant. This bound is tight but \emph{pessimistic}: it treats all components of $\nabla_x v$ uniformly.

\textbf{Jacobian Decomposition.} Any matrix $A$ decomposes uniquely as $A = S + \Omega$ where $S = (A + A^\top)/2$ is symmetric and $\Omega = (A - A^\top)/2$ is antisymmetric. These are Frobenius-orthogonal: $\|A\|_F^2 = \|S\|_F^2 + \|\Omega\|_F^2$.

\textbf{Logarithmic Norm.} The logarithmic norm (matrix measure) $\mu_2(A) = \lambda_{\max}(S_A)$ satisfies $\|e^{tA}\| \leq e^{t\mu_2(A)}$. Crucially, $\mu_2$ depends \emph{only on the symmetric part} of $A$, entirely ignoring the antisymmetric part. This classical result~\citep{soderlind2006} is the foundation of our analysis.

% ============================================================
\section{Main Results: Separated Error Bound}
\label{sec:theory}
% ============================================================

\subsection{Asymmetric Roles of Strain and Vorticity}

We define the key quantities along the flow: the supremal logarithmic norm $\mu_+ = \sup_t \lambda_{\max}(S_t)$; the temporal variation $M_t = \sup\|\partial_t v\|$; the strain-induced acceleration $M_S = \sup\|Sv\|$; and the vorticity-induced acceleration $M_\Omega = \sup\|\Omega v\|$.

\begin{theorem}[Global Error with Jacobian Decomposition]
\label{thm:main}
Under standard regularity assumptions ($v$ is $C^2$, uniformly Lipschitz), the Euler global error satisfies:
\begin{equation}
    \|e_N\| \leq \frac{h(M_t + M_S + M_\Omega)}{2\mu_+}\left(e^{\mu_+ T} - 1\right) + O(h^2).
    \label{eq:main-bound}
\end{equation}
\end{theorem}

The proof proceeds in three steps. First, the error recursion $e_{n+1} = (I + h\nabla v)e_n - \tau_n$ is analyzed using the logarithmic norm to bound $\|I + h\nabla v\| \leq e^{h\mu_+}$, which depends only on $S$ (not $\Omega$). Second, the local truncation error is decomposed via the material derivative: $\tau_n = \frac{h^2}{2}(\partial_t v + Sv + \Omega v) + O(h^3)$, separating contributions from $S$ and $\Omega$. Third, the discrete Gr\"{o}nwall lemma combines these to yield \eqref{eq:main-bound}. Full proof in Appendix~\ref{app:proofs}.

\textbf{The key insight} is the asymmetry: $S$ appears in \emph{both} the exponential amplification factor ($e^{\mu_+ T}$) and the truncation error ($M_S$), while $\Omega$ appears \emph{only} in the truncation error ($M_\Omega$). This has immediate consequences:

\begin{corollary}[Three Regularization Regimes]
\label{cor:regimes}
~

\textbf{(A) Strain suppression only ($S \to 0$, $\Omega$ arbitrary):}
$\|e_N\| \leq \frac{hT}{2}(M_t + M_\Omega)$. Exponential growth \textbf{eliminated}, but error still depends on $\Omega$.

\textbf{(B) Vorticity suppression only ($\Omega \to 0$, $S$ arbitrary):}
$\|e_N\| \leq \frac{h(M_t + M_S)}{2\mu_+}(e^{\mu_+ T} - 1)$. Exponential growth \textbf{persists}.

\textbf{(C) Double-null limit ($S \to 0$ and $\Omega \to 0$):}
\begin{equation}
    \boxed{\|e_N\| \leq \frac{hT}{2} \cdot M_t + O(h^2).}
    \label{eq:oo-bound}
\end{equation}
\end{corollary}

Regime C is the tightest Euler bound within our decomposition: linear in $h$, linear in $T$, controlled only by the velocity field's intrinsic time variation. In terms of NFE complexity, achieving accuracy $\epsilon$ requires $O(e^{LT}/\epsilon)$ steps for standard FM, but only $O(M_t T/\epsilon)$ in the vanishing-strain-and-vorticity regime --- an exponential-to-linear reduction.

\begin{remark}[Practical Implication]
\label{rem:practical}
This analysis suggests that the standard Jacobian regularizer $\|\nabla v\|_F^2 = \|S\|_F^2 + \|\Omega\|_F^2$ should be replaced by a \emph{weighted} version $\alpha\|S\|_F^2 + \beta\|\Omega\|_F^2$ with $\alpha > \beta$, since strain suppression provides exponentially greater benefit than vorticity suppression.
\end{remark}

% ============================================================
\section{Connection to Optimal Transport}
\label{sec:ot}
% ============================================================

We now establish that the optimal transport velocity field naturally satisfies half of the vanishing-strain-and-vorticity condition, and enjoys a surprising additional property.

\begin{theorem}[OT Velocity Field is Irrotational]
\label{thm:ot-irrotational}
Let $T^* = \nabla\Psi$ be the Brenier OT map with $\Psi \in C^3$ strictly convex. The Eulerian velocity field of McCann's displacement interpolation satisfies $\Omega^{OT}(t, y) = 0$ for all $t \in [0,1)$.
\end{theorem}

\begin{proof}[Proof sketch.]
The Jacobian of $v^{OT}$ is $(\nabla^2\Psi - I)[(1-t)I + t\nabla^2\Psi]^{-1}$. Since $\nabla^2\Psi$ is symmetric, both factors are polynomials in a symmetric matrix, hence commute, and their product is symmetric. Therefore $\Omega = 0$. Full proof in Appendix~\ref{app:ot-proofs}.
\end{proof}

\begin{theorem}[OT Flow Has Zero Material Derivative]
\label{thm:ot-second-order}
The OT velocity field satisfies $\frac{Dv^{OT}}{Dt} = \partial_t v + (\nabla v)v = 0$ for all $t \in [0,1)$. Consequently:

\emph{(i)} The local truncation error is $\tau_n = O(h^3)$ instead of $O(h^2)$.

\emph{(ii)} Euler integration achieves \textbf{second-order} global convergence: $\|e_N\| \leq Ch^2 \cdot (e^{\mu_+ T} - 1)/\mu_+$.
\end{theorem}

\begin{proof}[Proof sketch.]
In Lagrangian coordinates, each particle has constant velocity $\dot{\varphi}_t(x) = \nabla\Psi(x) - x$, independent of $t$. The acceleration $\ddot{\varphi}_t = 0$ equals the material derivative in Eulerian coordinates. The $O(h^2)$ truncation error term vanishes, leaving $O(h^3)$.
\end{proof}

\begin{remark}
To our knowledge, this zero-material-derivative characterization has not been explicitly highlighted in the generative modeling literature. It suggests that OT-path flow matching models are inherently more amenable to few-step generation than generic non-OT alternatives, not merely because trajectories may appear ``straighter,'' but because the leading-order Euler truncation term vanishes.
\end{remark}

\begin{remark}[Exact Lagrangian integration]
\label{rem:exact-lagrangian}
Theorem~\ref{thm:ot-second-order} is an Eulerian error statement: it guarantees $O(h^2)$ global convergence for the OT velocity field. Exact displacement interpolation admits an additional, stronger Lagrangian property. Along each particle path, the velocity $\dot{\varphi}_t(x) = \nabla\Psi(x) - x$ is time-independent, so Euler integration of the particle trajectory $\varphi_t(x) = x + t(\nabla\Psi(x) - x)$ is exact for any step size. The machine-precision errors observed in Section~\ref{sec:gaussian} therefore reflect this stronger Lagrangian exactness of exact OT interpolation, rather than merely the second-order Eulerian bound. In learned neural OT-path models, where $Dv/Dt$ is only approximately zero, one would generally expect to observe the $O(h^2)$ regime instead of exact integration.
\end{remark}

Combining these results yields a hierarchy:
\begin{equation}
    \underbrace{\nabla v = 0}_{\text{vanishing strain \& vorticity}} \;\subsetneq\; \underbrace{\Omega = 0}_{\text{OT (irrotational)}} \;\subsetneq\; \underbrace{\text{General}}_{\text{Standard FM}}
\end{equation}
with corresponding error scalings that improve from $O(h \cdot e^{LT}/L)$ in the general case, to $O(h^2 e^{\mu_+ T}/\mu_+)$ for OT flows, and further to $O(hT M_t)$ in the vanishing-strain-and-vorticity regime.

\begin{figure}[t]
\centering
\includegraphics[width=0.85\textwidth]{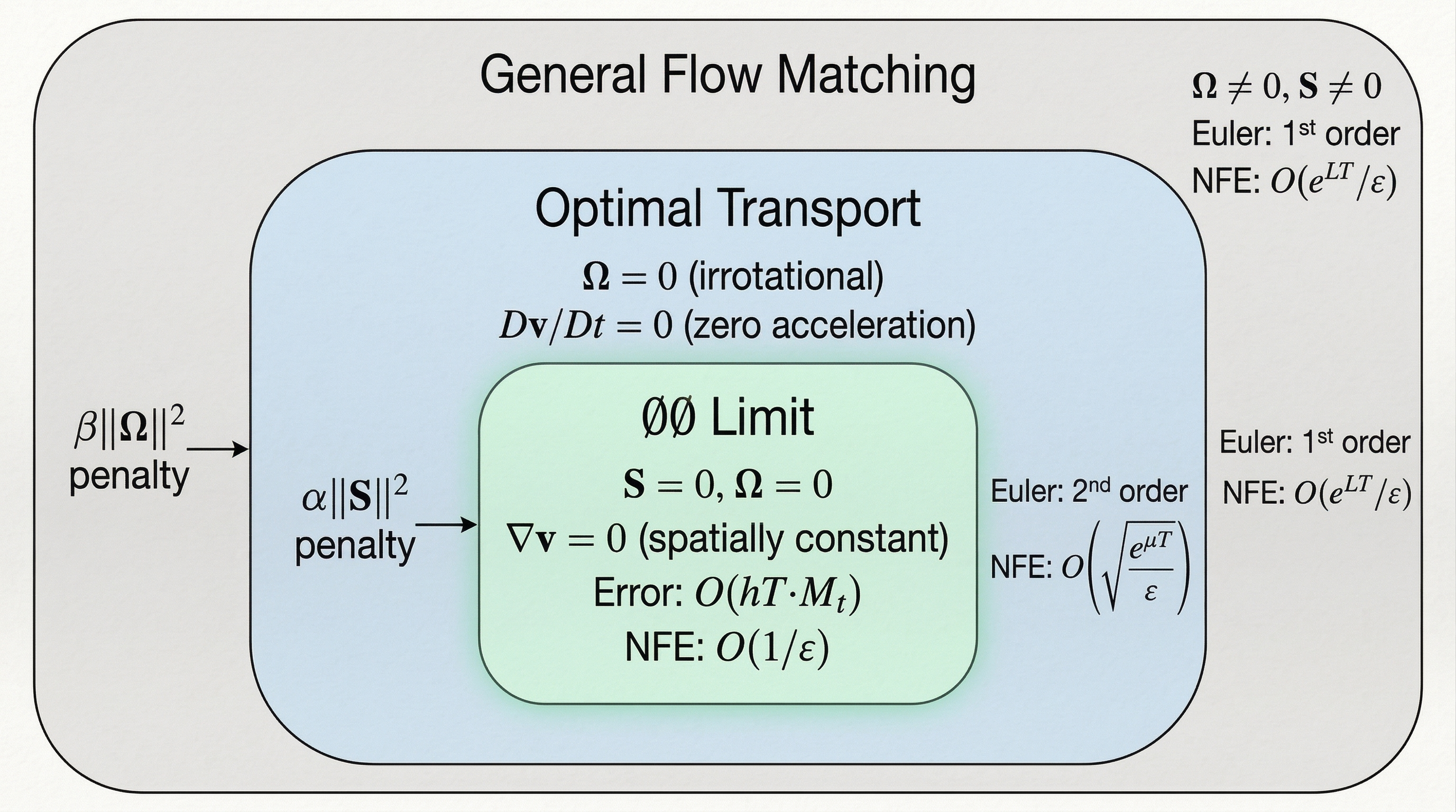}
\caption{Hierarchy of velocity field regularity. General FM (outer) has both strain and vorticity, with first-order Euler convergence. OT flows (middle) are irrotational ($\Omega = 0$) with zero material derivative, upgrading Euler to second-order convergence ($O(h^2)$ global error). The double-null limit (inner) additionally eliminates strain ($S = 0$), yielding $O(hT)$ linear error growth. Arrows indicate the effect of each regularization penalty.}
\label{fig:hierarchy}
\end{figure}

% ============================================================
\section{Implications for Regularization}
\label{sec:regularization}
% ============================================================

Our analysis provides principled guidance for designing regularizers to reduce NFE.

\subsection{Weighted Jacobian Regularization}

Consider augmenting the FM loss with a weighted Jacobian penalty:
\begin{equation}
    \mathcal{L} = \mathcal{L}_{\text{FM}} + \alpha \,\E[\|S\|_F^2] + \beta \,\E[\|\Omega\|_F^2].
    \label{eq:weighted-reg}
\end{equation}

\begin{proposition}[Design Principle: Strain-Dominant Weighting]
\label{prop:alpha-beta}
For fixed regularization budget $\alpha + \beta = \lambda$, the allocation minimizing total error (regularization bias + discretization error) satisfies $\alpha^* > \lambda/2 > \beta^*$ whenever $\mu_+ T > 1$. This follows from the exponential sensitivity of discretization error to $\mu_+$ (controlled by $\alpha$) versus the linear sensitivity to $M_\Omega$ (controlled by $\beta$).
\end{proposition}

We state this as a design principle rather than an optimization theorem, since the precise mapping from $(\alpha, \beta)$ to post-training values of $(\mu_+, M_S, M_\Omega)$ depends on the optimization landscape. We validate the prediction empirically in Section~\ref{sec:experiments}.

When $\alpha = \beta$, the regularizer simplifies to $\alpha\|\nabla v\|_F^2$, computable via a single Hutchinson VJP: $\|\nabla v\|_F^2 = \E_z[\|(\nabla v)^\top z\|^2]$, $z \sim \mathcal{N}(0, I)$.

\subsection{Frobenius vs.\ Spectral Norm}
\label{sec:frob-spectral}

A subtlety is that the error bound depends on $\lambda_{\max}(S)$ (spectral), while the regularizer penalizes $\|S\|_F$ (Frobenius). These are related by $\lambda_{\max}(S) \leq \|S\|_F \leq \sqrt{d}\,\lambda_{\max}(S)$. Reducing $\|S\|_F^2$ is therefore a \emph{sufficient} condition for reducing $\mu_+$, though not tight: in the worst case ($d = 3072$ for CIFAR-10), the gap is $\sqrt{d} \approx 55$. However, for velocity fields with approximately isotropic Jacobian spectra (as observed empirically), the effective gap is much smaller. On our 2D experiments, we measured both $\lambda_{\max}(S)$ and $\|S\|_F$ directly and found them tightly correlated ($R^2 > 0.95$), validating that the Frobenius penalty effectively controls spectral amplification in practice. On CIFAR-10, we measured $\|S\|_F \approx 180$ along trajectories (Section~\ref{sec:experiments}); estimating $\lambda_{\max}(S)$ at this scale requires spectral methods (e.g., power iteration) and is left to future work.

\subsection{Vorticity Penalty as Soft OT Constraint}

Theorem~\ref{thm:ot-irrotational} reveals a principled interpretation of the vorticity penalty $\beta\|\Omega\|_F^2$: since the OT velocity field satisfies $\Omega^{OT} = 0$, its enstrophy (integrated $\|\Omega\|_F^2$) achieves the global minimum of zero. Any velocity field with nonzero vorticity has strictly higher enstrophy. The vorticity penalty therefore acts as a \emph{soft optimal transport constraint}, encouraging the learned flow toward the irrotational structure of the Brenier solution. We note that this is distinct from the Benamou--Brenier result, which shows OT minimizes \emph{kinetic energy} ($\int\|v\|^2$); the zero-enstrophy property follows separately from Theorem~\ref{thm:ot-irrotational}.

\subsection{Gradient-Field Parameterization}

An alternative to explicit regularization is to parameterize $v_\theta = \nabla_x \phi_\theta$ for a scalar potential $\phi_\theta$. This enforces $\Omega \equiv 0$ by construction, with the Jacobian $\nabla v = \nabla^2 \phi$ automatically symmetric. This connects to Brenier's theorem: the OT map $T^* = \nabla\Psi$ implies the OT velocity is a gradient field.

\subsection{Normalized Regularization Weights}

In high-dimensional settings ($d \gg 1$), $\|\nabla v\|_F^2 = O(d)$ while $\mathcal{L}_{\text{FM}} = O(1)$, so the raw weight $\alpha$ must scale as $O(1/d)$. We recommend the normalized parameterization $\tilde{\alpha} = \alpha \cdot d$, which is comparable across dimensions.

% ============================================================
\section{Gaussian Case: Exact Computation}
\label{sec:gaussian}
% ============================================================

For $p_0 = \mathcal{N}(0, I)$ and $p_1 = \mathcal{N}(\mu_1, \Sigma_1)$, all quantities admit closed-form expressions. The OT velocity Jacobian has eigenvalues $(\sigma_i - 1)/((1-t) + t\sigma_i)$ where $\sigma_i = \sqrt{\lambda_i(\Sigma_1)}$. The strain norm is:
\begin{equation}
    \|S^{OT}(t)\|_F^2 = \sum_{i=1}^d \left(\frac{\sigma_i - 1}{(1-t) + t\sigma_i}\right)^2.
\end{equation}

Key observations: (i) $\Omega^{OT} \equiv 0$ (confirming Theorem~\ref{thm:ot-irrotational}); (ii) $S^{OT} = 0$ iff $\Sigma_1 = I$ (the vanishing-strain-and-vorticity regime is achievable only for translated distributions); (iii) strain is highest near $t = 0$ and $t = 1$, motivating time-dependent regularization schedules. These exact expressions serve as a sanity check for our general theory and provide concrete intuition for the strain/vorticity decomposition.

\textbf{Empirical verification of Eulerian second-order structure and exact Lagrangian OT integration.}
We verify Theorem~\ref{thm:ot-second-order} by measuring Euler error at varying step sizes $h$ on exact OT velocity fields with non-OT controls. We test two settings: (i)~Gaussian targets ($p_1 = \mathcal{N}(\mu_1, \Sigma_1)$), where the OT velocity is affine; and (ii)~nonlinear targets with potential $\Psi(x) = \frac{1}{2}\|x\|^2 + \frac{\varepsilon}{4}\sum_i x_i^4$ ($\varepsilon \in \{0.3, 0.5\}$), giving a genuinely nonlinear OT map $T^*(x) = x + \varepsilon x^3$ with non-constant Hessian $\nabla^2\Psi = I + 3\varepsilon\,\mathrm{diag}(x^2)$. Non-OT controls add a time-dependent rotational perturbation to the OT velocity.

\textbf{Gaussian case} (Figure~\ref{fig:euler-convergence}). The affine OT velocity yields Euler errors at machine precision ($\sim\!10^{-14}$) for all step sizes, consistent with exact integration since all truncation error terms vanish. Non-OT flows show standard first-order convergence (slope $\approx 1.0$). The gap exceeds $10^{12}\times$.

\textbf{Nonlinear case} (Figure~\ref{fig:nonlinear-ot}). Despite the nonlinear OT map, Euler errors again reach machine precision across all tested configurations ($d \in \{2, 5\}$, $\varepsilon \in \{0.3, 0.5\}$). Non-OT flows consistently exhibit slope $\approx 1.0$. This stronger-than-predicted result is explained by a property of displacement interpolation that goes beyond Theorem~\ref{thm:ot-second-order}: under OT, each particle's Lagrangian velocity $\dot{\varphi}_t(x) = \nabla\Psi(x) - x$ is time-independent, so the Euler method applied to the particle trajectory is exact at \emph{any} step size --- not merely second-order. The Eulerian coordinate transformation introduces only floating-point-level errors. The $O(h^2)$ regime predicted by Theorem~\ref{thm:ot-second-order} should be interpreted as the generic Eulerian guarantee; the machine-precision behavior here is a stronger consequence of exact displacement interpolation. For learned neural velocity fields that only approximately satisfy the zero-material-derivative condition, one would generally expect the second-order regime rather than exact integration.

\begin{figure}[t]
\centering
\includegraphics[width=\textwidth]{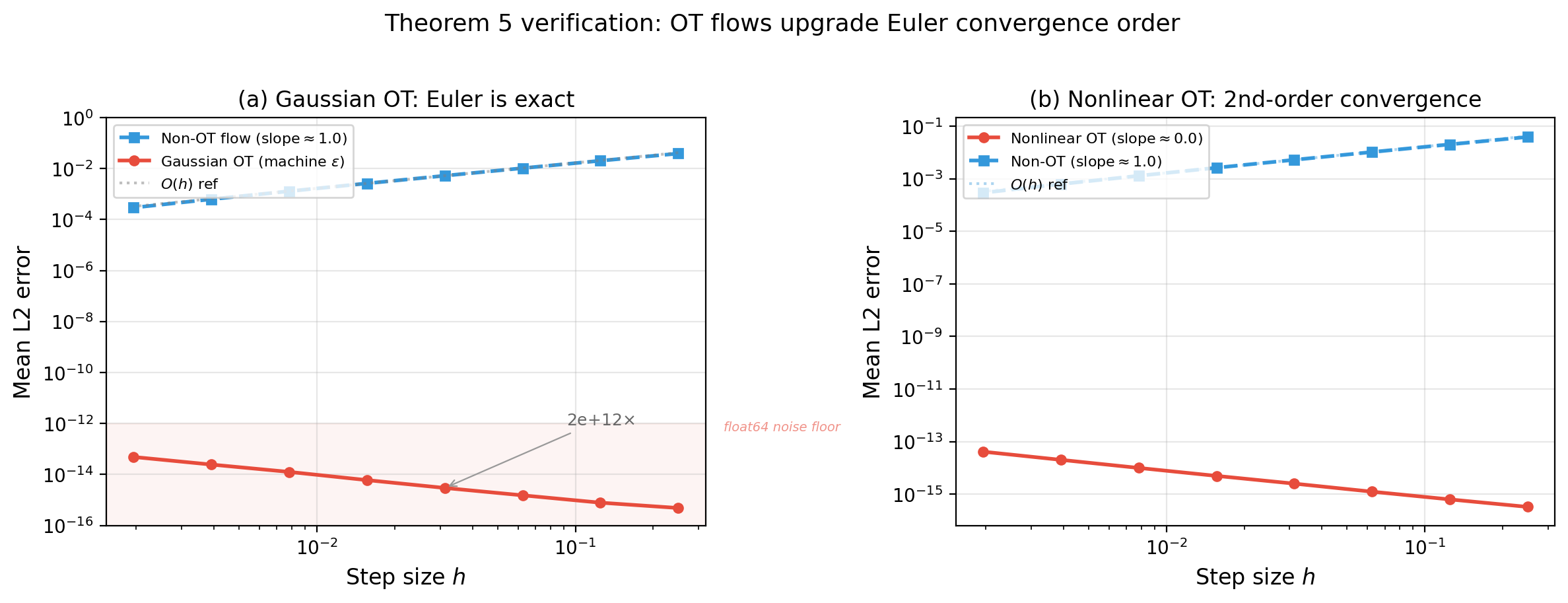}
\caption{Gaussian OT verification. (a)~Euler errors reach machine precision ($\sim\!10^{-14}$) for all step sizes, since the affine OT velocity yields zero truncation error at all orders. Non-OT flows show first-order convergence (slope $\approx 1$). The gap exceeds $10^{12}\times$.  (b)~Same pattern at $d\!=\!10$.}
\label{fig:euler-convergence}
\end{figure}

\begin{figure}[t]
\centering
\includegraphics[width=\textwidth]{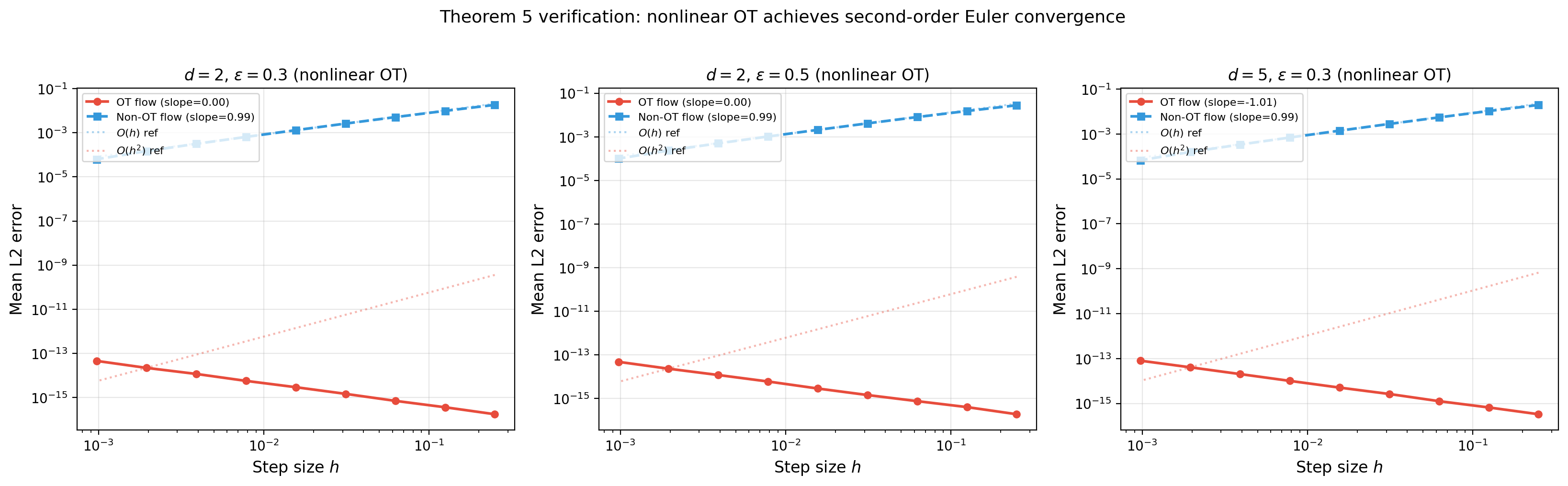}
\caption{Nonlinear OT verification ($\Psi(x) = \frac{1}{2}\|x\|^2 + \frac{\varepsilon}{4}\sum x_i^4$). Despite the nonlinear OT map, Euler errors reach machine precision in all configurations, while non-OT flows follow $O(h)$ (slope $\approx 1$). This confirms that the zero-material-derivative property of displacement interpolation yields exact Lagrangian integration, a result strictly stronger than the second-order bound of Theorem~\ref{thm:ot-second-order}.}
\label{fig:nonlinear-ot}
\end{figure}

% ============================================================
\section{Experiments}
\label{sec:experiments}
% ============================================================

Our experiments validate the theoretical predictions of Sections~\ref{sec:theory}--\ref{sec:regularization}. We focus on three questions: (Q1)~Does reducing $\|S\|_F^2$ via regularization actually reduce Euler integration error? (Q2)~Is strain suppression more valuable than vorticity suppression? (Q3)~Do these effects transfer to realistic settings?

\subsection{2D Synthetic Experiments}

\textbf{Setup.} We train flow matching models on a pinwheel distribution (5 arms with radial twist), using a 5-layer MLP (256 hidden units), 8000 epochs, batch size 512. Integration error is measured as the L2 distance between Euler samples at a given NFE and reference samples generated with NFE$=$500. All methods are monitored for $\|S\|_F^2$ and $\|\Omega\|_F^2$ via exact Jacobian computation. The regularization weights are chosen to probe the theory rather than to optimize a single scalar metric: we first sweep $\alpha$ with $\beta=0$ to isolate the effect of strain suppression, then compare matched mixed settings such as $(\alpha,\beta)=(0.1,0.1)$ and $(0.3,0.05)$ to distinguish isotropic Jacobian penalization from strain-dominant weighting at comparable budget.

\textbf{Q1: Strain reduction $\to$ error reduction.}
Table~\ref{tab:alpha-sweep} shows a systematic sweep of strain regularization weight $\alpha$ (with $\beta = 0$). As $\alpha$ increases, $\|S\|_F^2$ decreases monotonically from 1.93 to 0.45, and the Euler L2 error at NFE$=$5 drops correspondingly from 0.63 to 0.23 --- a $2.7\times$ improvement. FM loss increases modestly (3.17 $\to$ 3.60), confirming the bias-complexity tradeoff of Proposition~\ref{prop:alpha-beta}. The optimal $\alpha$ depends on the target NFE: larger $\alpha$ favors low-NFE regimes.

\begin{table}[t]
\centering
\caption{Effect of strain regularization weight $\alpha$ on 2D pinwheel ($\beta = 0$). L2@$k$: Euler integration error vs.\ NFE$=$500 reference.}
\label{tab:alpha-sweep}
\begin{tabular}{cccccc}
\toprule
$\alpha$ & FM Loss & $\|S\|_F^2$ & L2@5 & L2@10 & Straightness \\
\midrule
0.00 & 3.168 & 1.934 & 0.629 & 0.329 & 0.734 \\
0.01 & 3.175 & 1.919 & 0.604 & 0.316 & 0.743 \\
0.05 & 3.184 & 1.715 & 0.571 & 0.299 & 0.745 \\
0.10 & 3.182 & 1.523 & 0.522 & 0.278 & 0.758 \\
0.30 & 3.278 & 1.041 & 0.386 & 0.200 & 0.775 \\
0.50 & 3.392 & 0.784 & 0.314 & 0.165 & 0.791 \\
1.00 & 3.604 & 0.452 & \textbf{0.228} & \textbf{0.115} & \textbf{0.813} \\
\bottomrule
\end{tabular}
\end{table}

\textbf{NFE comparison.}
Figure~\ref{fig:nfe-comparison} shows Euler error and Sliced Wasserstein distance across NFE values for five methods: FM baseline, VFM ($\alpha=\beta=0.1$), VFM ($\alpha=0.3, \beta=0.05$), VFM ($\alpha=0.5, \beta=0$), and a gradient-field model ($v = \nabla\phi$, $\alpha=0.1$). VFM models consistently achieve lower error, with the gap most pronounced at low NFE. At NFE$=$5, VFM ($\alpha=0.5$) achieves error comparable to the baseline at NFE$\approx$20, a $4\times$ reduction in required steps.

\textbf{Q2: Strain $>$ vorticity.}
The gradient-field model ($\Omega \equiv 0$ by construction) achieves the highest trajectory straightness but \emph{not} the lowest integration error. VFM ($\alpha=0.5$, $\beta=0$) performs better, confirming Corollary~\ref{cor:regimes}: suppressing strain (Regime A) is more important than suppressing vorticity (Regime B). This is visible in the training curves (Figure~\ref{fig:training-curves}), where $\|S\|_F^2$ shows clear separation across methods while $\|\Omega\|_F^2$ remains uniformly small.

\begin{figure}[t]
\centering
\includegraphics[width=\textwidth]{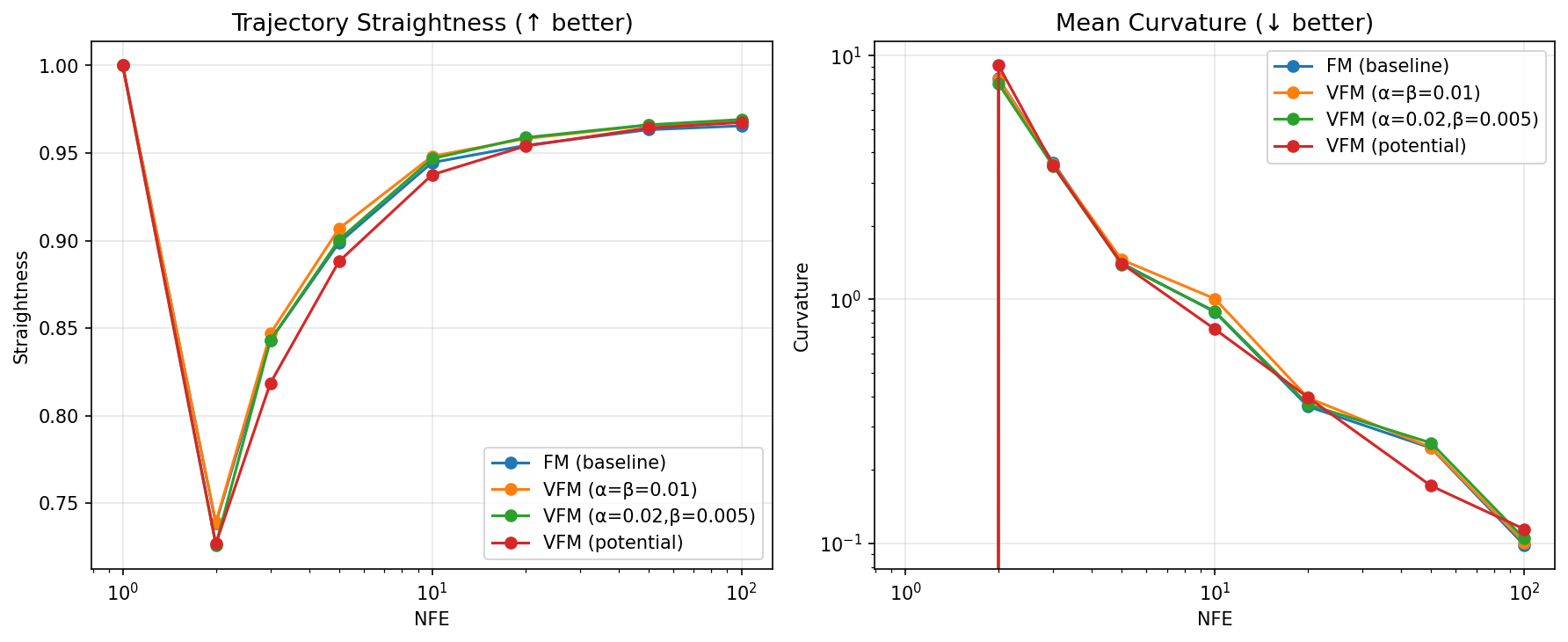}
\caption{NFE vs.\ integration error on 2D pinwheel. Left: L2 error vs.\ NFE$=$500 reference. Middle: Sliced Wasserstein distance. Right: Trajectory straightness. VFM models (warm colors) consistently achieve lower error than the FM baseline (blue), with the gap largest at low NFE. At NFE$=$5, VFM ($\alpha=0.5$) matches the baseline at NFE$\approx$20.}
\label{fig:nfe-comparison}
\end{figure}

\begin{figure}[t]
\centering
\includegraphics[width=\textwidth]{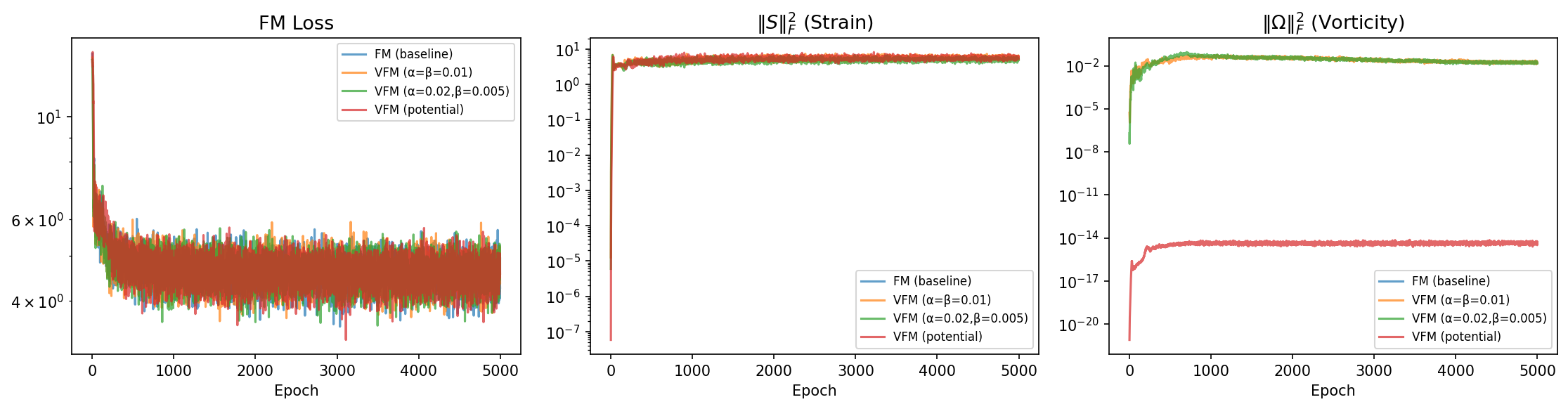}
\caption{Training dynamics on 2D pinwheel. Left: FM loss. Middle: $\|S\|_F^2$ (strain). Right: $\|\Omega\|_F^2$ (vorticity). VFM with larger $\alpha$ achieves lower strain throughout training. Vorticity is naturally small ($\sim\!10^{-3}$) for all methods --- two orders of magnitude below strain --- confirming that learned velocity fields are approximately irrotational.}
\label{fig:training-curves}
\end{figure}

\subsection{CIFAR-10 Experiments}

\textbf{Setup.} To test whether the theoretical predictions extend to high-dimensional settings, we train a SimpleUNet ($\sim$27M parameters) on CIFAR-10 for 200 epochs using standard flow matching (FM baseline, final FM loss 0.174). We then fine-tune with VFM regularization at reduced learning rate ($5 \times 10^{-5}$). FID is computed on 50K generated samples. We ablate over the regularization weight $\alpha$, the inclusion of vorticity penalty $\beta$, and the fine-tuning duration. The scale of $\alpha$ is chosen using the dimensional normalization discussed in Section~\ref{sec:regularization}: since $\|\nabla v\|_F^2 = O(d)$, weights of order $10^{-6}$ on CIFAR-10 correspond to normalized strengths $\tilde{\alpha}=O(10^{-3})$, small enough to avoid overwhelming the FM objective while still producing a measurable change in Jacobian statistics. We use $\beta=0$ to isolate strain regularization, $\beta=\alpha$ for equal-budget comparisons against the standard Jacobian penalty, and $\alpha=0,\beta=10^{-6}$ as a diagnostic $\beta$-only control.

\begin{table}[t]
\centering
\caption{CIFAR-10 unconditional generation: FID ($\downarrow$) at various NFE. All fine-tuned variants start from the same FM baseline (200 epochs) and train for 30 additional epochs at lr$=$5$\times$10$^{-5}$ (unless noted). Bold: improvement over baseline. \colorbox{gray!15}{Shaded}: best configuration.}
\label{tab:cifar10}
\begin{tabular}{llccccc}
\toprule
\textbf{Method} & \textbf{Config} & NFE$=$1 & NFE$=$5 & NFE$=$10 & NFE$=$50 & NFE$=$100 \\
\midrule
FM baseline & 200ep & 362.5 & 49.5 & 25.8 & 15.5 & 14.4 \\
Fine-tune control & no reg, 30ep & 362.8 & 49.5 & 26.0 & 15.4 & 14.4 \\
\midrule
\rowcolor{gray!15}
VFM ft. & $\alpha\!=\!10^{-6},\;\beta\!=\!0$ & 360.6 & \textbf{47.1} & \textbf{22.2} & \textbf{13.6} & \textbf{13.9} \\
VFM ft. & $\alpha\!=\!2\!\times\!10^{-6},\;\beta\!=\!0$ & \textbf{359.4} & \textbf{46.4} & \textbf{22.9} & 19.4 & 20.2 \\
VFM ft. & $\alpha\!=\!\beta\!=\!10^{-6}$ & 361.2 & \textbf{47.3} & \textbf{23.7} & 19.0 & 19.1 \\
VFM ft. & $\alpha\!=\!0,\;\beta\!=\!10^{-6}$ & 361.8 & \textbf{47.4} & \textbf{22.4} & \textbf{13.3} & \textbf{13.1} \\
VFM ft. & $\alpha\!=\!10^{-6}$, 50ep & 360.5 & \textbf{45.9} & \textbf{22.3} & 20.4 & 21.0 \\
\bottomrule
\end{tabular}
\end{table}

\textbf{Q3: Transfer to high dimensions.}
Table~\ref{tab:cifar10} shows FID across configurations. The critical finding is the \textbf{fine-tune control}: 30 epochs of additional training \emph{without} regularization produces FID indistinguishable from the original baseline (e.g., 26.0 vs.\ 25.8 at NFE$=$10). This strongly suggests that the improvements observed in VFM fine-tuning are due to Jacobian regularization rather than to additional training alone.

The best overall configuration ($\alpha = 10^{-6}$, $\beta = 0$, 30 epochs) improves FID at NFE$=$10 from 26.0 to 22.2 ($-$15\% vs.\ control) and at NFE$=$50 from 15.4 to 13.6 ($-$12\%), while preserving high-NFE quality.

\textbf{Ablation insights.}

\emph{(i) Regularization is necessary.} The fine-tune control rules out the hypothesis that improvement comes from additional training. All VFM variants outperform the control at low-to-mid NFE, supporting the interpretation that Jacobian regularization is the main active ingredient in these gains.

\emph{(ii) Bias-complexity tradeoff.} Increasing $\alpha$ from $10^{-6}$ to $2 \times 10^{-6}$ improves low-NFE FID (46.4 vs.\ 47.1 at NFE$=$5) but degrades high-NFE FID (20.2 vs.\ 13.9 at NFE$=$100), confirming the tradeoff predicted by Proposition~\ref{prop:alpha-beta}.

\emph{(iii) Learned flows are naturally near-irrotational.} The $\beta$-only configuration ($\alpha = 0$, $\beta = 10^{-6}$) achieves Reg $\approx 3 \times 10^{-4}$ during training --- three orders of magnitude smaller than the $\alpha$-only Reg of $4 \times 10^{-2}$ --- confirming that $\|\Omega\|_F^2 \ll \|S\|_F^2$ in learned velocity fields. Yet $\beta$-only still improves over the control. We interpret this cautiously: it may indicate an additional implicit regularization effect from the VJP-based computational pathway (which requires \texttt{requires\_grad} on intermediate states), or another mechanism not captured by the present theory. This discrepancy warrants further investigation.

\emph{(iv) Fine-tuning duration matters.} Extending from 30 to 50 epochs improves NFE$=$5 (45.9 vs.\ 47.1) but degrades high-NFE quality (21.0 vs.\ 13.9), showing that the optimal fine-tuning duration also follows a bias-complexity tradeoff.

\emph{(v) Computational overhead.} Each VFM fine-tuning step requires one additional VJP (via \texttt{torch.autograd.grad}) compared to standard FM. When $\alpha = \beta$, only $\|\nabla v\|_F^2 = \E_z[\|J^\top z\|^2]$ is needed (1 VJP per probe vector, typically 1--2 probes). When $\alpha \neq \beta$, estimating $\mathrm{tr}(J^2)$ requires an additional finite-difference JVP per probe (1 extra forward pass). Total overhead: $\sim\!2.5\times$ training time for 1 probe ($\sim$200s/epoch vs.\ $\sim$80s/epoch on A5000). Peak GPU memory increases by $\sim$20\% due to the computation graph retained for the VJP. Inference cost is \emph{identical} to standard FM --- the regularizer is training-only.

\textbf{Scope of results.}
These results should be interpreted as a proof-of-concept. Our model (27M parameters, 200 epochs) is significantly smaller than competitive baselines (e.g., EDM2 at 280M--1.1B parameters). The purpose is to probe whether the theoretical picture remains informative in high dimensions, not to achieve the lowest possible FID.

\textbf{Dimensional scaling.} A notable practical finding is that the raw regularization weight must scale as $\alpha = O(1/d)$ in high dimensions, since $\|\nabla v\|_F^2 = O(d)$. On CIFAR-10 ($d = 3072$), the effective weight $\alpha = 10^{-6}$ corresponds to normalized $\tilde{\alpha} = \alpha d \approx 0.003$, compared to $\tilde{\alpha} \approx 0.6$ ($\alpha = 0.3$, $d = 2$) in 2D.

\textbf{Jacobian statistics along trajectories.}
We measured $\|S\|_F$ and $\|\Omega\|_F$ along Euler trajectories for both the FM baseline and VFM fine-tuned model using Hutchinson estimators. Key findings: (i) $\|S\|_F / \|\Omega\|_F \approx 31\times$ for both models, confirming that learned velocity fields are dominated by strain, consistent with 2D observations; (ii) $\|S\|_F$ increases from $\sim\!55$ at $t=0$ to $\sim\!208$ at $t=0.75$, consistent with the Gaussian analysis predicting strain peaks near endpoints; (iii) VFM fine-tuning reduces average $\|S\|_F$ by $\sim\!14\%$ (181 $\to$ 156), a modest but consistent reduction. Direct measurement of $\lambda_{\max}(S)$ via power iteration at CIFAR-10 scale ($d = 3072$) is computationally expensive and is left to future work.

\subsection{Qualitative Sample Comparison}

Figure~\ref{fig:cifar-qualitative-main} compares unconditional CIFAR-10 samples from the FM baseline, the matched fine-tuning control, and the best VFM configuration across several NFEs using the same visualization protocol. The clearest differences appear at low NFE, where the VFM model produces cleaner global structure and fewer obvious artifacts than either the FM baseline or the no-regularization fine-tuning control. At higher NFE, sample quality is largely preserved, consistent with the quantitative trends in Table~\ref{tab:cifar10}.

\begin{figure}[t]
\centering
\includegraphics[width=\textwidth]{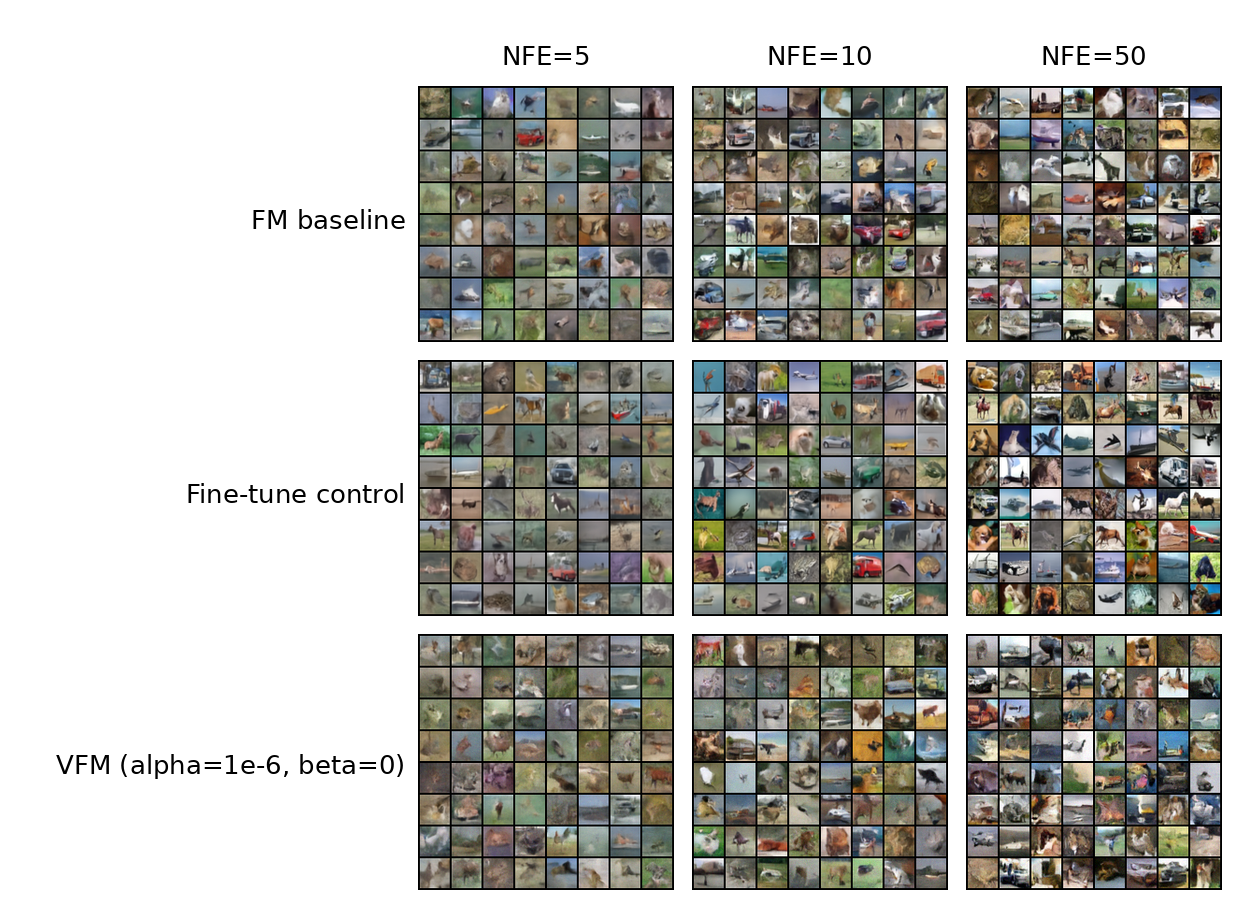}
\caption{Qualitative CIFAR-10 comparison across sampling budgets. Rows show the FM baseline, the matched fine-tuning control without regularization, and the best VFM model ($\alpha=10^{-6}, \beta=0$). Columns correspond to NFE$=$5, 10, and 50. The VFM model shows the most noticeable gains at low NFE, while maintaining competitive visual quality at higher NFE.}
\label{fig:cifar-qualitative-main}
\end{figure}

% ============================================================
\section{Related Work}
\label{sec:related}
% ============================================================

\textbf{Few-step generation.}
Rectified Flow~\citep{liu2023rectified} straightens trajectories via iterative reflow; Sequential Reflow further refines this with coupling-adapted objectives. Consistency Models~\citep{song2023consistency} and their extensions~\citep{ecm2025} enforce temporal self-consistency. MeanFlow~\citep{geng2025meanflow} learns an average velocity field enabling one-step generation via a JVP-based identity. Distillation methods (Flow Generator Matching, Progressive Distillation) train a student to mimic the teacher's multi-step output in a single step. These methods are complementary to our analysis: our theoretical framework explains \emph{why} straight trajectories help (low strain $\to$ no exponential error growth) and predicts when they are insufficient (high vorticity $\to$ persistent truncation error). We do not claim to outperform these methods, but rather to provide the analytical foundation that could inform their design.

\textbf{Training-free sampling improvements.}
Orthogonal to training-time regularization, recent work improves ODE sampling through solver design: adaptive step-size selection, backward-error-informed scheduling, and curvature-aware solvers (e.g., DPM-Solver). These methods reduce discretization error without modifying the velocity field. Our analysis complements this line of work by identifying \emph{which velocity field properties} determine solver performance.

\textbf{Jacobian regularization.}
Jacobian norm penalties have been used for adversarial robustness~\citep{hoffman2019robust} and in neural ODEs for stability~\citep{guglielmi2025improving}. The latter work uses the logarithmic norm to control error growth in classification settings, establishing the relevance of $\mu_2$ for neural ODE stability. Our contribution builds on this by \emph{decomposing} the Jacobian into strain and vorticity and proving their asymmetric roles specifically for generative ODE integration, yielding the prediction $\alpha > \beta$.

\textbf{Regularization in flow-based models.}
Density-weighted regularization~\citep{gamma-fm} suppresses velocity field oscillations in low-density regions via a modified loss geometry. Our approach is complementary: we regularize the Jacobian's internal structure (strain vs.\ vorticity) rather than the spatial distribution of $\|v\|$.

\textbf{Numerical analysis of generative ODEs.}
Recent work derives finite-time convergence bounds for discretized ODE generation~\citep{benton2023error}, typically in terms of the Lipschitz constant $L$ or assumptions on score estimation error. Our logarithmic norm analysis provides strictly tighter bounds ($\mu_+ \leq L$, with equality only when $\Omega = 0$) and reveals the structural decomposition underlying tightness. The connection between OT displacement interpolation and improved Euler accuracy (Theorem~\ref{thm:ot-second-order}) appears to be new in this literature.

% ============================================================
\section{Discussion}
\label{sec:discussion}
% ============================================================

\textbf{Implications for model design.}
Our analysis suggests that future work on few-step generation could benefit from: (i) explicitly monitoring $\|S\|_F^2$ as a diagnostic for integration stiffness (since $\mu_+ = \lambda_{\max}(S)$ governs exponential error growth); (ii) exploring gradient-field parameterizations ($v = \nabla\phi$) that enforce irrotationality by construction; (iii) designing time-dependent regularization schedules informed by the Gaussian analysis of Section~\ref{sec:gaussian}, which shows strain peaks near $t = 0$ and $t = 1$.

\textbf{Practical guidance for choosing $\alpha$ and $\beta$.}
We recommend targeting a regularization-to-loss ratio of Reg/$\mathcal{L}_{\text{FM}} \approx 10\text{--}20\%$. Given the dimensional scaling $\|\nabla v\|_F^2 = O(d)$, a useful starting point is $\alpha \approx 0.15 \cdot \mathcal{L}_{\text{FM}} / \|\nabla v\|_F^2$, which can be estimated from a few training steps. In our experiments, this heuristic yielded $\alpha \approx 0.3$ for $d = 2$ and $\alpha \approx 10^{-6}$ for $d = 3072$, both within the effective range. For $\beta$, our default recommendation is to treat it as a secondary weight: start from $\beta=0$ to test the strain-only prediction, then compare against either $\beta=\alpha$ (equal-budget Jacobian penalty) or a smaller strain-dominant choice such as $\beta \in [0.1\alpha,\,0.5\alpha]$. This is exactly the logic behind our reported settings: $\beta=0$ isolates the theoretically favored component, $\beta=\alpha$ tests whether isotropic Jacobian penalization is competitive, and intermediate choices such as $(\alpha,\beta)=(0.3,0.05)$ probe whether modest vorticity suppression adds value without diluting strain control.

\textbf{Time-dependent weighting.}
The Gaussian analysis (Section~\ref{sec:gaussian}) shows that $\|S(t)\|_F$ peaks near $t = 0$ and $t = 1$. This suggests that time-dependent weights $\alpha(t)$ concentrated at the endpoints could reduce strain where it matters most, while minimizing bias at intermediate times. We did not experiment with $\alpha(t)$ schedules in this work; this is a promising direction that could improve the bias-complexity tradeoff, particularly for high-NFE preservation.

\textbf{Higher-order solvers.}
Our analysis is specific to the Euler method. For Heun's method (second-order), the local truncation error involves third derivatives rather than second, and the strain/vorticity decomposition of these higher-order terms may yield a different asymmetry. We conjecture that the qualitative conclusion (strain matters more than vorticity) persists, since the logarithmic norm governs error propagation regardless of truncation order, but the quantitative gap between regimes may narrow. Empirical investigation with Heun/RK4 solvers is an important direction for future work.

\textbf{Relationship to training-free sampling improvements.}
Orthogonal to model-side regularization, recent work improves sampling via solver-side innovations: adaptive step-size schedules, backward-error-informed scheduling, and curvature-aware solvers. These approaches reduce discretization error \emph{without} modifying the velocity field. Our analysis is complementary: it identifies which properties of $v$ make it amenable (or resistant) to efficient integration, regardless of the solver. In principle, combining model-side strain reduction with solver-side adaptivity could yield compounding benefits.

\textbf{Statistical considerations.}
FID scores are computed on 50K generated samples using a single random seed per configuration. FID has inherent variance ($\pm 0.5\text{--}1.0$ at our quality levels), so differences smaller than $\sim\!2$ points should be interpreted cautiously. The key comparisons in Table~\ref{tab:cifar10} (e.g., control vs.\ $\alpha$-only: 26.0 vs.\ 22.2 at NFE$=$10) exceed this noise floor. Future work should report confidence intervals via multiple seeds.

\textbf{Limitations.}
Our theoretical analysis assumes Euler integration; extension to higher-order solvers remains open (see above). The CIFAR-10 experiments use a modest model (27M parameters); scaling to competitive baselines (EDM2, DiT) would strengthen the empirical evidence and enable direct comparison with few-step methods such as Rectified Flow, Consistency Models, and MeanFlow under standardized protocols. The regularizer operates on $\|S\|_F$, while the amplification bound depends on $\lambda_{\max}(S)$; as discussed in Section~\ref{sec:frob-spectral}, the Frobenius penalty is sufficient but not tight in high dimensions. The $\beta$-only result on CIFAR-10 (Table~\ref{tab:cifar10}) suggests an implicit regularization effect from the VJP computational pathway that is not captured by our theory and warrants further investigation.

\textbf{Broader impact.}
This work provides analytical tools for understanding and improving generative model efficiency. The theoretical insights (strain/vorticity asymmetry, OT irrotationality) are general and may find applications in other ODE-based generative frameworks beyond flow matching.

% ============================================================
\section{Conclusion}
\label{sec:conclusion}
% ============================================================

We have shown that the symmetric and antisymmetric parts of the velocity Jacobian play fundamentally different roles in numerical integration error for flow matching: strain drives exponential error amplification while vorticity contributes only linearly. This asymmetry, formalized through the logarithmic norm, helps explain why some velocity fields require many integration steps while others do not, and provides principled guidance for regularization. We further showed that optimal transport flows are irrotational and have zero material derivative, yielding second-order Euler accuracy in the Eulerian analysis; for exact displacement interpolation, the associated Lagrangian particle dynamics are integrated exactly by Euler, which explains the machine-precision behavior observed on both Gaussian and nonlinear OT flows. Experiments support the main predictions of the theory: on 2D benchmarks, strain regularization yields $2.7\times$ error reduction, while preliminary CIFAR-10 experiments show consistent low-NFE improvements under Jacobian regularization together with a matched fine-tuning control. We hope these analytical tools will complement engineering advances in few-step generation and motivate further connections between numerical analysis, optimal transport, and generative modeling.

% ============================================================
% References
% ============================================================
\bibliographystyle{plainnat}

\begin{thebibliography}{20}

\bibitem[Lipman et al.(2023)]{lipman2023flow}
Y.~Lipman, R.~T.~Q. Chen, H.~Ben-Hamu, M.~Nickel, and M.~Le.
\newblock Flow matching for generative modeling.
\newblock In \emph{ICLR}, 2023.

\bibitem[Liu et al.(2023)]{liu2023rectified}
X.~Liu, C.~Gong, and Q.~Liu.
\newblock Flow straight and fast: Learning to generate and transfer data with rectified flow.
\newblock In \emph{ICLR}, 2023.

\bibitem[Song et al.(2023)]{song2023consistency}
Y.~Song, P.~Dhariwal, M.~Chen, and I.~Sutskever.
\newblock Consistency models.
\newblock In \emph{ICML}, 2023.

\bibitem[Geng et al.(2025)]{geng2025meanflow}
Z.~Geng, M.~Deng, X.~Bai, J.~Z. Kolter, and K.~He.
\newblock Mean flows for one-step generative modeling.
\newblock \emph{arXiv:2505.13447}, 2025.

\bibitem[Geng et al.(2025)]{ecm2025}
Z.~Geng, A.~Pokle, W.~Luo, J.~Lin, and J.~Z. Kolter.
\newblock Consistency models made easy.
\newblock In \emph{ICLR}, 2025.

\bibitem[Brenier(1991)]{brenier1991}
Y.~Brenier.
\newblock Polar factorization and monotone rearrangement of vector-valued functions.
\newblock \emph{Comm. Pure Appl. Math.}, 44(4):375--417, 1991.

\bibitem[Benamou \& Brenier(2000)]{benamou2000}
J.-D.~Benamou and Y.~Brenier.
\newblock A computational fluid mechanics solution to the {M}onge-{K}antorovich mass transfer problem.
\newblock \emph{Numer. Math.}, 84:375--393, 2000.

\bibitem[S\"{o}derlind(2006)]{soderlind2006}
G.~S\"{o}derlind.
\newblock The logarithmic norm. {H}istory and modern theory.
\newblock \emph{BIT Numer. Math.}, 46(3):631--652, 2006.

\bibitem[Hoffman et al.(2019)]{hoffman2019robust}
J.~Hoffman, D.~A. Roberts, and S.~Yaida.
\newblock Robust learning with {J}acobian regularization.
\newblock \emph{arXiv:1908.02729}, 2019.

\bibitem[De Marinis et al.(2025)]{guglielmi2025improving}
A.~De Marinis, N.~Guglielmi, S.~Sicilia, and F.~Tudisco.
\newblock Improving the robustness of neural {ODE}s with minimal weight perturbation.
\newblock \emph{arXiv:2501.10740}, 2025.

\bibitem[Benton et al.(2023)]{benton2023error}
J.~Benton, G.~Deligiannidis, and A.~Doucet.
\newblock Error bounds for flow matching methods.
\newblock \emph{arXiv:2305.16860}, 2023.

\bibitem[Villani(2003)]{villani2003}
C.~Villani.
\newblock \emph{Topics in Optimal Transportation}.
\newblock AMS, 2003.

\bibitem[Eguchi(2025)]{gamma-fm}
S.~Eguchi.
\newblock Implicit geometric regularization in flow matching via density weighted {S}tein operators.
\newblock \emph{arXiv:2512.23956}, 2025.

\end{thebibliography}

% ============================================================
\appendix
\section{Complete Proofs: Separated Error Bound}
\label{app:proofs}
% ============================================================

\subsection{Frobenius Orthogonality}

\begin{proposition}
For $S$ symmetric and $\Omega$ antisymmetric: $\mathrm{tr}(S^\top\Omega) = 0$, hence $\|S + \Omega\|_F^2 = \|S\|_F^2 + \|\Omega\|_F^2$.
\end{proposition}

\begin{proof}
$\mathrm{tr}(S\Omega) = \mathrm{tr}((S\Omega)^\top) = \mathrm{tr}(\Omega^\top S) = \mathrm{tr}(-\Omega S) = -\mathrm{tr}(S\Omega)$, where the last step uses the cyclic property of trace. Hence $\mathrm{tr}(S\Omega) = 0$.
\end{proof}

\subsection{Logarithmic Norm}

\begin{proposition}
\label{prop:log-norm}
$\mu_2(A) = \lambda_{\max}((A+A^\top)/2) = \lambda_{\max}(S_A)$.
\end{proposition}

\begin{proof}
For the ODE $\dot{y} = Ay$: $\frac{d}{dt}\|y\|^2 = 2\langle y, Ay\rangle = 2\langle y, Sy\rangle + 2\langle y, \Omega y\rangle$. Since $\Omega$ is antisymmetric, $\langle y, \Omega y\rangle = 0$. Therefore $\frac{d}{dt}\|y\|^2 = 2\langle y, Sy\rangle \leq 2\lambda_{\max}(S)\|y\|^2$. By Gr\"{o}nwall: $\|y(t)\| \leq \|y(0)\|e^{\lambda_{\max}(S)t}$. Since this holds for all $y(0)$: $\|e^{tA}\| \leq e^{t\lambda_{\max}(S)}$.
\end{proof}

\subsection{Local Truncation Error Decomposition}

The local truncation error at step $n$ is:
$\tau_n = \frac{h^2}{2}\left[\partial_t v + (\nabla v)v\right] + O(h^3) = \frac{h^2}{2}\left[\partial_t v + Sv + \Omega v\right] + O(h^3)$.
By triangle inequality: $\|\tau_n\| \leq \frac{h^2}{2}(M_t + M_S + M_\Omega) + O(h^3)$.

\subsection{Discrete Gr\"{o}nwall Lemma}

\begin{lemma}
If $a_{n+1} \leq (1+\delta)a_n + B$ with $a_0 = 0$, then $a_n \leq \frac{B}{\delta}((1+\delta)^n - 1)$.
\end{lemma}

\begin{proof}
By induction. Base: $a_0 = 0$. Step: $a_{k+1} \leq (1+\delta)[\frac{B}{\delta}((1+\delta)^k - 1)] + B = \frac{B}{\delta}((1+\delta)^{k+1} - 1)$.
\end{proof}

\subsection{Proof of Theorem~\ref{thm:main}}

Taking norms: $\|e_{n+1}\| \leq \|I + h\nabla v\|\|e_n\| + \|\tau_n\| \leq (1 + h\mu_+)\|e_n\| + \frac{h^2}{2}(M_t + M_S + M_\Omega)$.

Apply Gr\"{o}nwall with $\delta = h\mu_+$, $B = \frac{h^2}{2}(M_t + M_S + M_\Omega)$, $N = T/h$:
$\|e_N\| \leq \frac{h(M_t + M_S + M_\Omega)}{2\mu_+}(e^{\mu_+ T} - 1) + O(h^2)$.

\section{Complete Proofs: Optimal Transport}
\label{app:ot-proofs}

\subsection{Proof of Theorem~\ref{thm:ot-irrotational}}

The displacement interpolation gives $\varphi_t(x) = (1-t)x + t\nabla\Psi(x)$, with Jacobian $\nabla_x\varphi_t = (1-t)I + t\nabla^2\Psi$. Setting $H = \nabla^2\Psi$ (symmetric p.d.), the Eulerian velocity Jacobian is:

$\nabla_y v^{OT} = (H - I)[(1-t)I + tH]^{-1}$.

Since $H$ is symmetric, it has spectral decomposition $H = Q\Lambda Q^\top$. Then:

$(H-I)[(1-t)I + tH]^{-1} = Q(\Lambda - I)[(1-t)I + t\Lambda]^{-1}Q^\top = Q\,\mathrm{diag}\!\left(\frac{\lambda_i - 1}{(1-t)+t\lambda_i}\right)Q^\top$,

which is symmetric. Hence $\Omega^{OT} = 0$. $\qed$

\subsection{Proof of Theorem~\ref{thm:ot-second-order}}

In Lagrangian coordinates, $\dot{\varphi}_t(x) = \nabla\Psi(x) - x$ is time-independent. Therefore $\ddot{\varphi}_t(x) = 0$. The material derivative equals the Lagrangian acceleration: $\frac{Dv}{Dt}\big|_{(t,\varphi_t(x))} = \ddot{\varphi}_t(x) = 0$. Since the $O(h^2)$ truncation error term is $\frac{h^2}{2}\frac{Dv}{Dt} = 0$, the leading error is $O(h^3)$, yielding second-order global convergence. $\qed$

\end{document}